\documentclass{article}
\usepackage{ijcai24}

\usepackage{times}
\usepackage{soul}
\usepackage{url}
\usepackage[hidelinks]{hyperref}
\usepackage[utf8]{inputenc}
\usepackage[small]{caption}
\usepackage{graphicx}
\usepackage{amsmath}
\usepackage{amsthm}
\usepackage{booktabs}
\usepackage{algorithm}
\usepackage{algorithmic}
\usepackage[switch]{lineno}

\usepackage[nameinlink,capitalise]{cleveref}
\usepackage{acronym}
\usepackage{paralist}
\usepackage{amsfonts}
\usepackage{subfig}
\usepackage{amssymb}

\newtheorem{definition}{Definition}[section]

\newtheorem{theorem}[definition]{Theorem}

\crefname{definition}{Definition}{Definitions}
\crefname{theorem}{Theorem}{Theorems}
\crefname{lemma}{Lemma}{Lemmas}
\crefname{corollary}{Corollary}{Corollaries}
\crefname{example}{Example}{Examples}

\acrodef{RV}[RV]{random variable}
\acrodef{VE}[VE]{variable elimination}
\acrodef{BE}[BE]{bucket elimination}
\acrodef{BP}[BP]{belief propagation}
\acrodef{JT}[JT]{junction tree algorithm}
\acrodef{TN}[TN]{tensor network}
\acrodef{BTN}[BTN]{base tensor network}
\acrodef{BT}[BT]{base tensor}
\acrodef{DT}[DT]{default tensor}
\acrodef{PGM}[PGM]{probabilistic graphical model}
\acrodef{PS}[PS]{parametric structure}
\acrodef{TD}[TD]{tensor decomposition}
\acrodef{EFG}[EFG]{extended factor graph}
\acrodef{EF}[EF]{extended factor}
\acrodef{ARV}[ARV]{artificial random variable}
\acrodef{SF}[SF]{split factor}
\acrodef{MT}[MT]{modification tensor}
\acrodef{RL}[RL]{reinforcement learning}
\acrodef{FG}[FG]{factor graph}
\acrodef{MDP}[MDP]{Markov decision process}
\acrodef{SRV}[SRV]{symmetric RV}

\title{Combining Local Symmetry Exploitation and Reinforcement Learning for Optimised Probabilistic Inference -- A Work In Progress}

\author{
Sagad Hamid$^1$
\and
Tanya Braun$^1$
\affiliations
$^1$University of Münster
\emails
\{sagad.hamid, tanya.braun\}@uni-muenster.de
}

\begin{document}

\maketitle

\begin{abstract}
Efficient probabilistic inference by variable elimination in graphical models requires an optimal elimination order. However, finding an optimal order is a challenging combinatorial optimisation problem for models with a large number of random variables. Most recently, a reinforcement learning approach has been proposed to find efficient contraction orders in tensor networks. Due to the duality between graphical models and tensor networks, we adapt this approach to probabilistic inference in graphical models. Furthermore, we incorporate structure exploitation into the process of finding an optimal order. Currently, the agent's cost function is formulated in terms of intermediate result sizes which are exponential in the number of indices (i.e., random variables). We show that leveraging specific structures during inference allows for introducing compact encodings of intermediate results which can be significantly smaller. By considering the compact encoding sizes for the cost function instead, we enable the agent to explore more efficient contraction orders. The structure we consider in this work is the presence of local symmetries (i.e., symmetries within a model's factors).
\end{abstract}

\section{Introduction}
Probabilistic inference tasks, such as querying the marginals of \acp{RV} or the partition function, show up consistently in various research areas including computer vision \cite{wang2013markov}, robotics \cite{dellaert2017factor}, natural language processing \cite{chater2006probabilistic}, and causality \cite{darwiche2022causal}. \Acp{PGM} take advantage of (conditional) independencies among \acp{RV} to compactly encode joint probability distributions as factorisations \cite{koller2009probabilistic}. Classical inference algorithms like \ac{VE} \cite{zhang1996exploiting} exploit these independencies to perform more efficient inference in \acp{PGM}. For models with a large number of \acp{RV}, efficiency heavily relies on an optimal elimination order as it determines the intermediate result sizes (and thus the computational complexity) required during inference. Since finding an optimal elimination order is a challenging combinatorial optimisation problem, heuristics are usually used instead \cite{koller2009probabilistic}.

In the research field of \ac{TN} contraction, an analogous problem exists where the goal is to find an optimal contraction order. Here, a \ac{RL} approach was recently introduced that formulates the problem as a \ac{MDP} to find efficient contraction orders in \acp{TN} \cite{meirom2022optimizing}. At the core, an \ac{RL} agent learns to select contraction orders s.t.\ the cumulative contraction costs are minimised. Similar to the case for probabilistic inference, each contraction cost is related to the size of the intermediate tensor that is created during contraction. We refer to the paper for more details on the setting.

\paragraph{Contributions}
We adapt the aforementioned \ac{RL} approach to probabilistic inference in \acp{PGM} since there exists a duality between \acp{PGM} and \acp{TN} \cite{robeva2019duality}. Further, we extend the approach by incorporating structure exploitation into the process of finding an optimal order. When dealing with large \acp{PGM}/\acp{TN}, it is often the case that some factors/tensors yield specific structures in their potentials/values. However, currently the agent's cost function does not consider these structures since it is solely defined in terms of intermediate result sizes which are exponential in the respective number of \acp{RV}/indices. Crucially, many structures allow for introducing compact encodings for intermediate results that can be significantly smaller. Considering the compact encoding sizes for the cost function instead allows the agent to explore new orders which can be more efficient. The structure we consider in this work is the presence of local symmetries, i.e., symmetries among \acp{RV} within a model's factors. To enable incorporating the compact encodings of local symmetries into the cost function, we show how to exploit local symmetries during inference while preserving correctness.

\paragraph{Structure Exploitation \& Efficient Inference}
In the context of \acp{PGM}, structure exploitation has been successfully incorporated in many different ways for more efficient inference. The research field of statistical relational artificial intelligence mainly deals with the exploitation of global symmetries (i.e., symmetries among \acp{RV} in the underlying joint probability distribution) and isomorphisms in a relational context. Several lifted models have been introduced to compactly encode these symmetries \cite{kimmig2015lifted}. Lifted inference algorithms perform inference on lifted models to exploit symmetries during inference, allowing for drastically reducing computational complexity \cite{niepert2014tractability,ludtke2018state,holtzen2020generating}. In other contexts, various structures and encodings have been considered for more efficient inference including tensor structures \cite{obermeyer2019tensor}, Fourier representations \cite{DBLP:conf/icml/XueEBGS16}, causal mechanisms \cite{DBLP:conf/ecai/Darwiche20,darwiche2022causal}, polynomial representations \cite{DBLP:conf/nips/WinnerS16,DBLP:conf/icml/WinnerSS17}, and circuit-based representations \cite{choi2020probabilistic}.

\section{Preliminaries}
\label{section:background}

In this section, we first specify notations and introduce fundamentals on \acp{FG} as a general \ac{PGM} formalism and exact inference by \ac{VE}. \acp{RV} are assumed to be discrete. For brevity, we only consider calculating the partition function. For space reasons, we refer to \cite{meirom2022optimizing} for background on tensors, \acp{TN}, and \ac{TN} contraction.

We denote RVs by uppercase letters \(X\), their domain by \(\text{Dom}(X)\), and their values by lowercase letters \(x \in \text{Dom}(X)\). Further, we write $x$ for $X = x$. W.l.o.g.\ we assume $\text{Dom}(X) = \{0,...,m-1\}$ for a \ac{RV} $X$ that can take $m$ values. Boldface uppercase letters \(\boldsymbol{X}\) and boldface lowercase letters \(\boldsymbol{x}\) denote sets of RVs and their joint assignment, respectively.

Let \(\boldsymbol{X}\) be a set of \(n\) \acp{RV}. An \ac{FG} $G = (\boldsymbol{X} \cup \boldsymbol{F}, \boldsymbol{E})$ is an undirected bipartite graph consisting of variable nodes $\boldsymbol{X}$, factor nodes $\boldsymbol{F}$, and edges \(\boldsymbol{E} \subseteq \boldsymbol{X} \times \boldsymbol{F}\). Each variable node represents an \ac{RV} \(X \in \boldsymbol{X}\) and each factor node a non-negative function \(\phi: \boldsymbol{Y} \rightarrow \mathbb{R}_+\) over a subset \(\boldsymbol{Y} \subseteq \boldsymbol{X}\) mapping assignments \(\boldsymbol{y} \in \text{Dom}(\boldsymbol{Y})\) to potentials \(\phi(\boldsymbol{y}) \in \mathbb{R}_+\).

A factor node is connected to a variable node if the \ac{RV} appears as an argument in the factor (see \cref{fig:example_fg1}). The graph structure encodes a joint probability distribution \(P(\boldsymbol{X})\) by
\begin{align*}
P(\boldsymbol{X}) = \frac{1}{Z(G)}\prod_{i=1}^{k}\phi_{i}(\boldsymbol{Y}_i),
\end{align*}
with \(Z(G)\) called the partition function of \(G\).
Instead of operating on the exponentially large distribution \(P(\boldsymbol{X})\), inference is performed on the factorised representation using the following two operations, multiplication of factors and sum out of \acp{RV}. For \(\boldsymbol{U},\boldsymbol{V} \subseteq \boldsymbol{X},\) multiplying two factors \(\phi_i(\boldsymbol{U})\) and \(\phi_j(\boldsymbol{V})\) results in a new factor \(\phi_k(\boldsymbol{W}),\) where \(\boldsymbol{W} = \boldsymbol{U} \cup \boldsymbol{V}\) and \(\phi_k(\boldsymbol{w}) = \phi_i(\boldsymbol{u}) \cdot \phi_j(\boldsymbol{v})\) with \(\pi_{\boldsymbol{U}}(\boldsymbol{w}) = \boldsymbol{u}\) and \(\pi_{\boldsymbol{V}}(\boldsymbol{w}) = \boldsymbol{v}\) (\(\pi_{(\boldsymbol{\cdot})}\) projects its input to the subscript \acp{RV}). For a factor \(\phi(\boldsymbol{Y}) \in \boldsymbol{F}\), summing out \acp{RV} \(\boldsymbol{Y}_1 \subseteq \boldsymbol{Y}\) yields a new factor \(\phi'(\boldsymbol{Y} \backslash \boldsymbol{Y}_1) = \sum_{\boldsymbol{Y}_1}\phi(\boldsymbol{Y})\).
The following theorems are used to sum out \acp{RV} efficiently.
\begin{theorem}
\label{theorem:VESumOut1}
For factor \(\phi(\boldsymbol{Y})\) with \(\boldsymbol{Y}_1\),\(\boldsymbol{Y}_2 \subseteq \boldsymbol{Y}\), and \(\boldsymbol{Y}_1 \cap \boldsymbol{Y}_2 = \emptyset\), we have
\(\sum_{\boldsymbol{Y}_1}\sum_{\boldsymbol{Y}_2}\phi(\boldsymbol{Y}) = \sum_{\boldsymbol{Y}_2}\sum_{\boldsymbol{Y}_1}\phi(\boldsymbol{Y})\).
\end{theorem}
\begin{theorem}
\label{theorem:VESumOut2}
For factors \(\phi_i(\boldsymbol{U}),\phi_j(\boldsymbol{V})\) with \(\boldsymbol{W} \subseteq \boldsymbol{V}, \boldsymbol{W} \cap \boldsymbol{U} = \emptyset\), we have
\(\sum_{\boldsymbol{W}} \phi_i(\boldsymbol{U}) \cdot \phi_j(\boldsymbol{V}) = \phi_i(\boldsymbol{U}) \cdot \sum_{\boldsymbol{W}} \phi_j(\boldsymbol{V})\).
\end{theorem}
Calculating \(Z(G)\) requires summing out all \(n\) \acp{RV}. \Cref{theorem:VESumOut1} implies that \acp{RV} can be summed out in any order, while \Cref{theorem:VESumOut2} yields that summing out any chosen \ac{RV} \(X\) requires multiplying only factors in which \(X\) appears as an argument. This multiplication  creates an intermediate result factor with a size being exponential in the number of all distinct \acp{RV} included during multiplication. For inference, the \ac{VE} algorithm \cite{zhang1996exploiting} multiplies factors successively to sum out \acp{RV} one by one. The efficiency clearly depends on the order in which \acp{RV} are summed out, as different orders result in different intermediate factor sizes. However, finding an optimal order is NP-hard \cite{koller2009probabilistic} and heuristics are used instead.

\begin{figure}[t]
\centering
\subfloat[]{%
\label{fig:example_fg1}%
\includegraphics[scale=0.5]{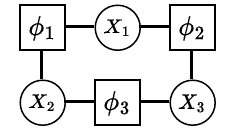}%
}
\subfloat[]{%
\label{fig:example_tn}%
\includegraphics[scale=0.5]{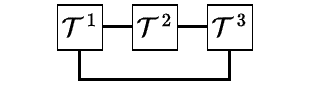}%
}
\subfloat[]{%
\label{fig:example_fg2}%
\includegraphics[scale=0.5]{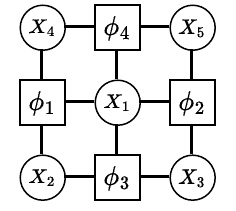}%
}
\caption{\textbf{(a)} A factor graph representing a joint distribution \(P(X_1,X_2,X_3)\). \textbf{(b)} A \ac{TN} with three shared indices and three tensors (labels for indices are omitted). \textbf{(c)} A factor graph for which there is no \ac{TN} with a one-to-one correspondence.}
\label{fig:exemplary_fg_and_tn}
\end{figure}

\begin{table}
    \centering
    \resizebox{0.29\textwidth}{!}{%
    \begin{tabular}{c |c}
        \hline
        \textbf{\ac{PGM}} & \textbf{\ac{TN}} \\
        \hline
        \ac{RV} & Index \\
        Factor & Tensor  \\
        Sum Out & Contraction \\
        Factor Multiplication & Tensor Product \\
        Domain Size & Dimension Size \\
        \hline
    \end{tabular}
    }
    \caption{Terminology for \acp{PGM} and \acp{TN}.}
\label{table: terminology_fg_tn}
\end{table}

\section{Duality between PGMs and TNs}
Since discrete joint distributions can be interpreted as tensors, there exist a duality between \acp{PGM} and \acp{TN} \cite{robeva2019duality}. \acp{PGM} encode joint distribution as factorisations, while \acp{TN} encode tensors in terms of tensor contractions. To illustrate the correspondence, consider \cref{fig:example_fg1,fig:example_tn} depicting an \ac{FG} with factors $\phi_1(X_1,X,2),\phi_2(X_1, X_3),\phi_3(X_2,X_3),$ and an \ac{TN} with tensors $\mathcal{T}^1, \mathcal{T}^2, \mathcal{T}^3$. An overview of the respective terminology is given in \cref{table: terminology_fg_tn}. Each factor node $\phi_i$ corresponds to a tensor node $\mathcal{T}_i$ and each \ac{RV} $X_i$ to an edge since edges represent indices. A factor's size is given by the product of the domain sizes of its \acp{RV}, while a tensor's size is given by the product of the dimensions of its indices. Contracting an index that connects two tensors in the \ac{TN} requires that we first multiply the corresponding tensors, which results in a new tensor that includes all joint indices, and then sum over the shared index. This translates in the \ac{FG} to summing out a \ac{RV} by multiplying all factors containing the \ac{RV} followed by summing over the assignments of the \ac{RV}. Consequently, contracting all shared indices in \cref{fig:example_tn} (which is referred to as \ac{TN} contraction) coincides with calculating the partition function in \cref{fig:example_fg1}. Note that the usual definition of \acp{TN} restricts indices to appear in at most two tensors. In contrast, \acp{RV} can appear in arbitrary many factors. Consequently, not every \ac{FG} has a one-to-one correspondence to a \ac{TN} (see \cref{fig:example_fg2}). The correspondence between \acp{PGM} and \acp{TN} allows for adapting the \ac{MDP} formulation presented in \cite{meirom2022optimizing} to the problem of finding an optimal elimination order. 

\section{MDP Formulation \& RL Setting}
For brevity, we refer to \cite{van2012reinforcement} for background on \acp{MDP} and \ac{RL}. Analogous to the problem of finding an optimal contraction order in \acp{TN}, we consider the problem of finding an optimal elimination order in \acp{FG} and formulate it as an MDP.
For an \ac{FG} $G=(\boldsymbol{X} \cup \boldsymbol{F}, \boldsymbol{E})$ with $n$ \acp{RV} $\boldsymbol{X} = \{X_1,\dots,X_n\}$, an elimination order $\rho = (X_{i_1},\dots,X_{i_n})$ defines in which order we sum out \acp{RV} to calculate the partition function. Summing out an \ac{RV} $X \in \boldsymbol{X}$ creates an intermediate factor $\phi^*(\boldsymbol{X}^*\backslash\{X\}) = \prod_{\phi(\boldsymbol{X}_i) \in \boldsymbol{F}_{X}} \phi(\boldsymbol{X}_i)$ with $\boldsymbol{F}_X = \{\phi(\boldsymbol{X}_i) \in \boldsymbol{F} \mid X \in \boldsymbol{X}_i\}$ and $\boldsymbol{X}^* = \bigcup \boldsymbol{X}_i$. Since a factor over \acp{RV} $\boldsymbol{X}^*$ needs to be created to obtain $\phi^*(\boldsymbol{X}^*\backslash\{X\})$, we generally define the sum-out cost as
$c(X)= |\text{Dom}(\boldsymbol{X}^*)|$
and the total elimination order cost as
$c(\rho) = \sum_{j=1}^{n-1} c(X_j)$.
Please note that we obtain different sets of \acp{RV} $\boldsymbol{X}^*$ for different orders.
Similarly, we aim for an order $\rho^{\text{min}}$ of minimal cost, i.e.,
\begin{align}
    \mathcal{\rho}^{\text{min}}(G) = \text{argmin}_{\rho}~c(\rho).
\end{align}
Having defined the problem of finding an optimal elimination order, we analogously formulate the corresponding \ac{MDP} $(\mathsf{S}, \mathsf{A}, \mathsf{F}, \mathsf{R})$: The state space $\mathsf{S} = \{G=(\boldsymbol{X} \cup \boldsymbol{F}, \boldsymbol{E})\}$ is the space of all \acp{FG}, the actions $\mathsf{A} = \boldsymbol{X}$ is the set of \acp{RV} to sum out, the transition function $\mathsf{F}(G,X)=G'(\boldsymbol{X}' \cup \boldsymbol{F}', \boldsymbol{E}')$ takes the current \ac{FG} and the \ac{RV} to sum out as input and outputs the resulting \ac{FG} $G'$ with $\boldsymbol{X}'= \boldsymbol{X}^*\backslash\{X\}, \boldsymbol{F}'=\boldsymbol{F}\backslash\boldsymbol{F}_X\cup\phi^*(\boldsymbol{X}^*\backslash\{X\})$, and the (immediate) cost function $\mathsf{R}(G,X)=c(X)$ is set to be the cost of summing out $X$ in $G$. With the MDP formulation at hand, the goal is now to learn a policy that chooses actions, i.e., an \ac{RV} to sum out next, s.t. the cumulative sum-out cost is kept minimal. To do so, we can analogously introduce a graph neural network for \ac{RV} selection. Specifically, we use the graph neural network to determine the next \ac{RV} to sum out for a given \ac{FG} to sequentially eliminate all \acp{RV}. After each elimination, we remove the node of the eliminated \ac{RV} and the nodes of the corresponding factors and include a new factor, defined over the \acp{RV} involved, that represents the intermediate result. To calculate the partition function, we continue until all \acp{RV} are eliminated, though it can be adapted to any type of query. After each elimination step, we update the total cost according to the cost function. For space reasons, the detailed description of the \ac{RL} procedure is left for future work. 

We observe that in the current formalism the sum-out cost is defined w.r.t.\ the domain size of $\boldsymbol{X}^*$ as we calculate an intermediate result factor $\phi^*(\boldsymbol{X}^*)$. Unfortunately, this size is exponential in the number of \acp{RV} $|\boldsymbol{X}^*|$. However, if we know beforehand that the intermediate factor has some structure that allows for introducing a compact encoding, we do not have to calculate all potentials, but only those required for the compact encoding. As a consequence, we can reformulate the cost function to not depend exponentially on $|\boldsymbol{X}^*|$ anymore, but rather on the size of the compact encoding. With costs that scale with specific structures, we enable the agent to explore new elimination orders. In the following, we show how local symmetries can be compactly encoded and how probabilistic inference can benefit from these compact encodings to create smaller intermediate factors.

\section{Exploitation of Local Symmetries}
Let $\phi(\boldsymbol{X})$ be a factor over $n$ \acp{RV} $\boldsymbol{X}$. We say that $\phi(\boldsymbol{X})$ has local symmetries if there exist interchangeable \acp{RV} $\boldsymbol{\bar{X}} \subseteq \boldsymbol{X}$, i.e., for any joint assignment $\boldsymbol{x}$ any permutation $\sigma_{\boldsymbol{\bar{x}}}(\boldsymbol{x})$ of assignments $\boldsymbol{\bar{x}} = \pi_{\boldsymbol{\bar{X}}}(\boldsymbol{x})$ does yield the same potential:
\begin{align}
\label{eq:local_symmetry}
\phi(\boldsymbol{x}) = \phi(\sigma_{\boldsymbol{\bar{X}}}(\boldsymbol{x})).
\end{align}
Note that a symmetry requires all \acp{RV} $\bar{X} \in \boldsymbol{\bar{X}}$ to have the same domain size. We say a factor is symmetric if all \acp{RV} $\boldsymbol{X}$ are interchangeable. To illustrate the structure, consider the symmetric factor $\phi$ depicted in \cref{fig:symmetric_factor}. We observe that all assignments with the same number of ones (being permutations of each other) are mapped to identical potentials, yielding four unique potentials in total. Please note that it it possible that multiple symmetries are present, i.e., that the \acp{RV} of multiple disjoint subsets of $\boldsymbol{X}$ are interchangeable within the respective subset. We denote the set of all disjoint subsets of interchangeable \acp{RV} in $\phi(\boldsymbol{X})$ by $\text{Sym}_{\phi}(\boldsymbol{X})$. A symmetric factor yields $\text{Sym}_{\phi}(\boldsymbol{X}) = \{\boldsymbol{X}\}$.

\begin{figure}[t!]
\centering
\resizebox{0.3\textwidth}{!}{%
\subfloat[]{
\label{fig:symmetric_factor}
\begin{tabular}{c c c c}
\toprule
$X_1$ & $X_2$ & $X_3$ & $\phi$ \\
\midrule
 0 & 0 & 0 & $1$ \\
 0 & 0 & 1 & $2$ \\
 0 & 1 & 0 & $2$ \\
 0 & 1 & 1 & $3$ \\
 1 & 0 & 0 & $2$ \\
 1 & 0 & 1 & $3$ \\
 1 & 1 & 0 & $3$ \\
 1 & 1 & 1 & $4$ \\
 \bottomrule
\end{tabular}
}
\qquad
\subfloat[]{
\label{fig:compact_encoding}
\begin{tabular}{c c}
\toprule
$S_{\{X_1,X_2,X_3\}}$ & $\phi^s$ \\
\midrule
 $[0,3]$ & $1$ \\
 $[1,2]$ & $2$ \\
 $[2,1]$ & $3$ \\
 $[3,0]$ & $4$ \\
 \bottomrule
\end{tabular}
}
}
\caption{A symmetric factor and its compact encoding.}
\label{fig:symmetric_factor_with_compact_encoding}
\end{figure}

We continue with compactly encoding symmetries within factors. Our encoding is motivated by Counting \acp{RV} which have been introduced in the context of lifted inference \cite{milch2008lifted}. For each set of interchangeable \acp{RV} $\boldsymbol{\bar{X}} \in \text{Sym}_{\phi}(\boldsymbol{X})$, we introduce a \ac{SRV} $S_{\boldsymbol{\bar{X}}}$ whose domain $\text{Dom}(S_{\boldsymbol{\bar{X}}})$ consists of histograms that count how many \ac{RV} are set to specific values. To illustrate those histograms, consider \cref{fig:compact_encoding} which represents the compact encoding $\phi^s$ of the symmetric factor $\phi$ in \cref{fig:symmetric_factor}. Each value of $S_{\{X_1,X_2,X_3\}}$ is a histogram that is a representative for all possible permutations of those values with the respective counts. More specifically, since \acp{RV} $X_1,X_2,X_3$ are Boolean, the histograms are of length two with the first entry counting zero assignments and the second entry counting one assignments. We observe that the four entries in $\phi^s$ suffice to represent all potentials in $\phi$. Consequently, taking symmetries into account enables reducing a factor's size drastically. By means of combinatorics, it follows immediately that 
\begin{align}
|\text{Dom}(S_{\boldsymbol{\bar{X}}})| = \binom{n'+d'-1}{d'-1},    
\end{align}
with $\bar{n} = |\boldsymbol{\bar{X}}|$ and $\bar{d} = |\text{Dom}(\bar{X})|$ for any $\bar{X} \in \boldsymbol{\bar{X}}$.

Next, we show how to perform the common factor operations, namely multiplication of factors and sum-out of \acp{RV}, with the compact encoding. This allows us to take advantage of local symmetries to introduce compact encodings of intermediate result factors during inference. We start with the sum-out operation which preserves symmetries:

\begin{theorem}[\textbf{Sum-Out with Symmetries}]
\label{theorem:symmetric_sum_out}
Let $\phi(\boldsymbol{X})$ be a factor and let $X \in \boldsymbol{X}, \boldsymbol{X}'=\boldsymbol{X}\backslash\{X\}$, and $\phi'(\boldsymbol{X}') = \sum_{X}\phi(\boldsymbol{X})$. Let further $\text{Sym}'_{\phi}(\boldsymbol{X}) = \text{Sym}_{\phi}(\boldsymbol{X})\backslash\{\boldsymbol{\bar{X}}\} \cup \{\boldsymbol{\bar{X}}\backslash\{X\}\}$ if there exists $\boldsymbol{\bar{X}} \in \text{Sym}_{\phi}(\boldsymbol{X})$ with $X \in \boldsymbol{\bar{X}}$. Then, the following holds:
\[ 
\text{Sym}_{\phi'}(\boldsymbol{X}') = 
\begin{cases} 
    \text{Sym}_{\phi}(\boldsymbol{X}) & \text{if } \nexists \boldsymbol{\bar{X}} \in \text{Sym}_{\phi}(\boldsymbol{X}): X \in \boldsymbol{\bar{X}} \\
    \text{Sym}'_{\phi}(\boldsymbol{X})  & \text{else }
\end{cases}
\]
\end{theorem}

\begin{proof}
Summing out \ac{RV} $X$ with $d = |\text{Dom}(X)|$ yields for any $\boldsymbol{x}' \in \text{Dom}(\boldsymbol{X}')$:
\begin{align*}
\phi'(\boldsymbol{x}') = \sum_{i=0}^{d-1}\phi(\boldsymbol{x}',i).        
\end{align*}
For the first case, i.e., where $X$ is not included in any interchangeable subset in $\text{Sym}_{\phi}(\boldsymbol{X})$, we obtain $\forall \boldsymbol{\bar{X}} \in \text{Sym}_{\phi}(\boldsymbol{X})$:
\begin{align*}
     \phi'(\boldsymbol{x}') = \sum_{i=0}^{d-1}\phi(\boldsymbol{x}',i) = \sum_{i=0}^{d-1}\phi(\sigma_{\boldsymbol{\bar{x}}}(\boldsymbol{x}'),i) = \phi'(\sigma_{\boldsymbol{\bar{x}}}(\boldsymbol{x}'))
\end{align*}
for any permutation $\sigma_{\boldsymbol{\bar{X}}}(\boldsymbol{x}')$, meaning that symmetries hold in $\phi'(\boldsymbol{X}')$ as well. It is clear that the same argument holds for the latter case with interchangeable \acp{RV} in $\text{Sym}'_{\phi}(\boldsymbol{X})$.
\end{proof}

An immediate consequence of \cref{theorem:symmetric_sum_out} is that in the presence of symmetries the sum-out operation can be performed with the compact encoding. Most importantly, since the resulting factor preserves symmetries, we only have to calculate as many potentials as required for the resulting compact encoding. We continue with the multiplication of factors with symmetries:

\begin{theorem}[\textbf{Multiplication with Symmetries}]
\label{theorem:symmetric_mult}
Let $\phi_1(\boldsymbol{X}_1), \phi_2(\boldsymbol{X}_2), \phi_3(\boldsymbol{X}_3)$ be factors with $\phi_3(\boldsymbol{X}_3) = \phi_1(\boldsymbol{X}_1) \cdot \phi_2(\boldsymbol{X}_2)$. Then, the following holds for $i,j \in \{1,2\}$ with $i \neq j$:
\begin{enumerate}[(i)]
    \item $\boldsymbol{\bar{X}}_i \in \text{Sym}_{\phi_i}(\boldsymbol{X}_i) \Rightarrow$ $\boldsymbol{\bar{X}}_i\backslash(\boldsymbol{\bar{X}}_i \cap \boldsymbol{X}_j) \in \text{Sym}_{\phi_3}(\boldsymbol{X}_3)$
    \item $\boldsymbol{\bar{X}}_i \in \text{Sym}_{\phi_i}(\boldsymbol{X}_i) \land \exists \boldsymbol{\bar{X}}_j \in \text{Sym}_{\phi_j}(\boldsymbol{X}_j): \boldsymbol{\bar{X}}_i \cap \boldsymbol{\bar{X}}_j \neq \emptyset\\ \Rightarrow \boldsymbol{\bar{X}}_i \cap \boldsymbol{\bar{X}}_j \in \text{Sym}_{\phi_3}(\boldsymbol{X}_3)$
\end{enumerate}
\end{theorem}

\begin{proof}
For assignments $\boldsymbol{x}_3 \in \boldsymbol{X}_3$ we have
\begin{align*}
    \phi_3(\boldsymbol{x}_3) = \phi_i(\boldsymbol{x}_{3i})\cdot\phi_j(\boldsymbol{x}_{3j}),
\end{align*}
with $\boldsymbol{x}_{3i} = \pi_{\boldsymbol{X}_i}(\boldsymbol{x}_3), \boldsymbol{x}_{3j} = \pi_{\boldsymbol{X}_j}(\boldsymbol{x}_3)$. For case (i), let $\boldsymbol{\bar{X}}_i' = \boldsymbol{\bar{X}}_i\backslash(\boldsymbol{\bar{X}}_i \cap \boldsymbol{X}_j)$ for any $\boldsymbol{\bar{X}}_i \in \text{Sym}_{\phi_i}(\boldsymbol{X}_i)$. Since $\boldsymbol{\bar{X}}'_i$ does not contain any \acp{RV} in $\boldsymbol{X}_j$, we obtain
\begin{align*}
    \phi_3(\boldsymbol{x}_3) &= \phi_i(\boldsymbol{x}_{3i})\cdot\phi_j(\boldsymbol{x}_{3j}) \\&= \phi_i(\sigma_{\boldsymbol{\bar{x}}'_i}(\boldsymbol{x}_{3i}))\cdot\phi_j(\boldsymbol{x}_{3j}) \\&= \phi_3(\sigma_{\boldsymbol{\bar{x}}'_i}(\boldsymbol{x}_3)),
\end{align*}
 for any permutation $\sigma_{\boldsymbol{\bar{x}}_i'}(\boldsymbol{x}_3)$. For case (ii), let $\boldsymbol{\bar{X}}_{ij} = \boldsymbol{\bar{X}}_i \cap \boldsymbol{\bar{X}}_j$ with $\boldsymbol{\bar{X}}_i \cap \boldsymbol{\bar{X}}_j \neq \emptyset$ for any two sets $\boldsymbol{\bar{X}}_i \in \text{Sym}_{\phi_i}(\boldsymbol{X}_i), \boldsymbol{\bar{X}}_j \in \text{Sym}_{\phi_j}(\boldsymbol{X}_j)$. Then, we obtain 
 \begin{align*}
    \phi_3(\boldsymbol{x}_3) &= \phi_i(\boldsymbol{x}_{3i})\cdot\phi_j(\boldsymbol{x}_{3j}) \\&= \phi_i(\sigma_{\boldsymbol{\bar{x}}_{ij}}(\boldsymbol{x}_{3i}))\cdot\phi_j(\sigma_{\boldsymbol{\bar{x}}_{ij}}(\boldsymbol{x}_{3j})) \\&= \phi_3(\sigma_{\boldsymbol{\bar{x}}_{ij}}(\boldsymbol{x}_3)),
 \end{align*}
 for any permutation $\sigma_{\boldsymbol{\bar{x}}_{ij}}(\boldsymbol{x}_3)$.
\end{proof}
\cref{theorem:symmetric_mult} provides information on which symmetries remain after multiplication, allowing us to determine the compact encoding of the resulting factor.

\begin{figure}[t!]
\centering
\includegraphics[scale=0.5]{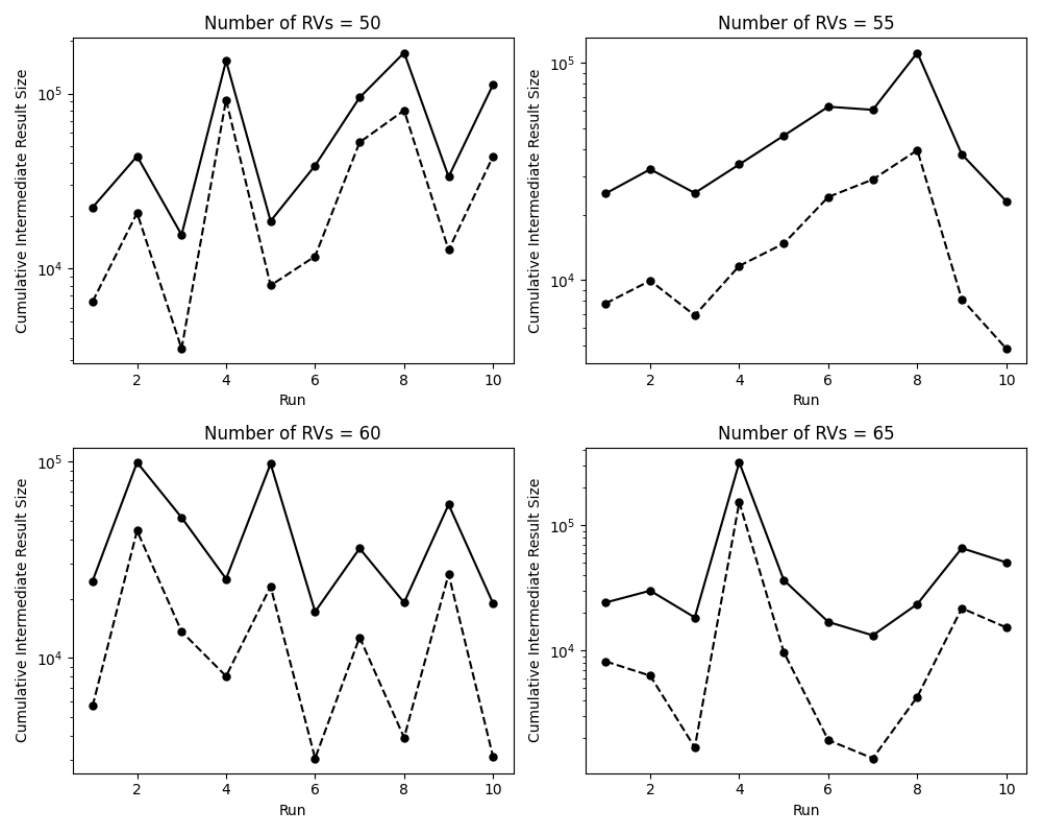}%
\caption{Cumulative intermediate result sizes for \ac{VE} without (solid line) and with (dashed line) compact encodings based on the presence of local symmetries.}
\label{fig:sym_eval}
\end{figure}

To show how exploiting local symmetries affects the cumulative costs, we have experimentally evaluated the cumulative intermediate result sizes with and without compact encodings for \ac{VE}. The results are shown in \cref{fig:sym_eval}. We have created $10$ random \acp{FG} for four settings each with $10$ symmetric factors and $50$/$55$/$60$/$65$ Boolean \acp{RV}. Each factor contains between $5$ and $10$ \acp{RV}. For each \ac{FG} (run), we have calculated the partition function where we select the elimination order based on the factor product that creates the smallest number of \acp{RV}. We clearly observe in all settings that structure exploitation allows for significantly reducing the total costs for specific models. With redefined costs based on compact encodings that scale with different properties rather than the number of \acp{RV}, we enable the agent to consider new elimination orders which previously might not have been considered due to higher costs. Please note that considering compact encodings does always reduce or retain a sum-out cost, meaning that in the worst-case (i.e., if no local symmetries are present) we obtain the same costs as without considering them.
\section{Conclusion}
We have adapted a recently introduced \ac{RL} approach for finding optimal contraction orders in \acp{TN} to the problem of finding optimal elimination orders for probabilistic inference in \acp{PGM}. We further incorporated the exploitation of local symmetries which allows for introducing compact encodings for intermediate result factors, enabling the agent to explore new elimination orders.
There exist many other interesting compact encodings such as low-rank tensor decompositions that can be taken into account for the cost function. Further, one can reformulate the \ac{RL} setting to consider evidence as well as other types of queries.

\bibliographystyle{named}
\bibliography{ijcai24}

\end{document}